\documentclass[pdflatex,sn-mathphys-num]{sn-jnl}% Math and Physical Sciences Numbered Reference Style
\usepackage{graphicx}%
\usepackage{multirow}%
\usepackage{amsmath,amssymb,amsfonts}%
\usepackage{amsthm}%
\usepackage{mathrsfs}%
\usepackage[title]{appendix}%
\usepackage{xcolor}%
\usepackage{textcomp}%
\usepackage{manyfoot}%
\usepackage{booktabs}%
\usepackage{algorithm}%
\usepackage{algorithmicx}%
\usepackage{algpseudocode}%
\usepackage{listings}%
\usepackage{lmodern}%
\theoremstyle{thmstyleone}%
\newtheorem{theorem}{Theorem}%  meant for continuous numbers
\newtheorem{proposition}[theorem]{Proposition}% 

\theoremstyle{thmstyletwo}%
\newtheorem{remark}{Remark}%
\newtheorem{cor}{Corollary}
\theoremstyle{thmstylethree}%

\begin{document}

\title[Network disorder and the turning distance]{Polygonal network disorder and the turning distance}

%%=============================================================%%
%% GivenName	-> \fnm{Joergen W.}
%% Particle	-> \spfx{van der} -> surname prefix
%% FamilyName	-> \sur{Ploeg}
%% Suffix	-> \sfx{IV}
%% \author*[1,2]{\fnm{Joergen W.} \spfx{van der} \sur{Ploeg} 
%%  \sfx{IV}}\email{iauthor@gmail.com}
%%=============================================================%%

\author[1]{\fnm{Alexander} \sur{Dolce}}

\author[2]{\fnm{Ryan} \sur{Lavelle}}

\author[2]{\fnm{Bernard} \sur{Scott}}

\author[1]{\fnm{Ashlyn} \sur{Urbanski}}

\author*[1]{\fnm{Joseph} \sur{Klobusicky}}\email{joseph.klobusicky@scranton.edu}

\affil[1]{\orgdiv{Department of Mathematics}, \orgname{University of Scranton}, \orgaddress{\street{800 Linden St.}, \city{Scranton}, \postcode{18510}, \state{Pennsylvania}, \country{USA}}}

\affil[2]{\orgdiv{Department of Computer Science}, \orgname{University of Scranton}.}

%%==================================%%
%% Sample for unstructured abstract %%
%%==================================%%

\abstract{ 

The turning distance is a well-studied metric for measuring the similarity between two polygons.  This metric is constructed by taking an $L^p$ distance between step functions which track each shape's tangent angle of a path tracing its boundary. In this study, we introduce \textit{turning disorders} for polygonal planar networks, defined by averaging turning distances between network faces and ``ordered" shapes (regular polygons or circles).   We derive closed-form expressions of turning distances for special classes of regular polygons, related to the divisibility of $m$ and $n$, and also between regular polygons and circles.  These formulas are used to show that the time for computing the  2-turning distances reduces to $O((m+n) \log(m+n))$ when both shapes are regular polygons, an improvement from $O(mn\log(mn))$ operations needed to compute distances between general polygons of $n$ and $m$ sides. We also apply these formulas to several examples of network microstructure with varying disorder.  For Archimedean lattices, a class of regular tilings, we can express turning disorders with exact expressions. We also consider turning disorders applied to two examples of stochastic processes on networks: spring networks evolving under T1 moves and polygonal rupture processes.  We find that the two aspects of defining different turning disorders, the choice of ordered shape and whether to apply area-weighting, can capture different notions of network disorder.
}

\keywords{shape recognition, turning distance, microstructure, T1 move}

%%\pacs[JEL Classification]{D8, H51}

%%\pacs[MSC Classification]{35A01, 65L10, 65L12, 65L20, 65L70}

\maketitle

\section{Introduction}

In the field of computer vision, a popular topic of research focuses on how to measure the similarity between two shapes.  For example, measures have been derived for comparing shapes to commonly found structures such as cubes \cite{martinez2009measuring}, rectangles \cite{rosin1999measuring}, and circles \cite{reeve1979calibration}.  In two dimensions, Arkin et al.\ \cite{arkin1989efficiently} provided a general metric, the \textit{turning distance}, for comparing shapes by using $L^p$ distances between \textit{turning functions} which track tangent vectors along paths tracing each shape's boundary. In polygons, turning functions are step functions, and it was found that turning distance under the $L^2$ norm can be computed  with a run time of $O(mn \log(mn))$ for comparing two polygons with $m$ and $n$ sides.  In this paper, we will employ the turning function to define a notion of disorder in planar polygonal networks.  We will frequently refer to such networks as microstructure, since approximately polygonal networks often appear in the study of grain boundary coarsening of polycrystalline metals in materials science \cite{henseler2003reduced,epshteyn2022stochastic,elsey2015mean}.  However, such patterns in nature are found on a variety of scales, ranging from pottery crazing \cite{rice2015pottery} to Martian terrains \cite{seibert2001small}.  

While the notion of ``disorder" can have multiple interpretations,  here we will consider disorder of a network as an average of turning distances, taken between cells of the network and some predetermined ``ordered shape".  For a polygon $P$ with $n$ sides,  the three choices we will use for an  ordered shape are the circle $C$, the regular hexagon $R_6$, and a regular polygon $R_n$ with the same number of sides as $P$.  For each of these measures, we also consider disorders with area-weighted averaging. While the various types of disorder appear to differ minimally in their definitions, we nonetheless find that these measures can vary widely given different examples of microstructures.  In general, however, all of the disorders assign higher disorder to microstructures with nonconvex and irregular cells.

In Sec.\ \ref{sec:turndef}, we review the turning distance, and subsequently introduce the various turning disorders, along with some of their basic properties.  We delay an analysis of examples of the turning disorder, and instead focus first on formulas of turning distances where one of the shapes is either a regular polygon or a circle.  Sec.\ \ref{sec:compturn} gives computationally efficient formulas for computing the turning distance.  In the case of comparing two regular polygons with $n$ and $m$ sides, the turning distance formula can be simplified to a summation formula needing only $O((m+n) \log(m+n))$ operations. For Sec.\ \ref{sec:compspec}, further simplifications can be made when considering the divisibility between the number of sides of the regular polygons.  In several cases, the turning distance can be written in closed form.  A simple summation formula is also provided for comparing any polygon to a circle.  This also reduces to closed form when the circle is compared to a regular polygon.  

The formulas derived in Secs.\ \ref{sec:turndef} and \ref{sec:compturn} can be used for efficiently computing several examples of microstructure.  In Sec.\ \ref{sec:arch}, we consider an ordered example of Archimedean lattices, a class of microstructures comprised of tiles which are regular polygons. Archimedean lattices arise in nanoscale science, including problems in self-assembly \cite{kulkarni2020archimedean} and quantum magnets \cite{farnell2014quantum}.   For soap foams with millimeter scale cells, Bae et al. \cite{bae2019controlled} gave methods for generating two-dimensional foams with Archimedean lattice structure. We write explicit expressions for each of the turning distance for the $(4,8^2), (4,6,12),$ and $(3, 12^2)$ lattices.  Sec.\ \ref{sec:dis} studies two classes of stochastic processes on networks that produce increasingly disordered structures.  In Sec.\ \ref{subsec:t1}, we consider linear spring networks in force balance.  To evolve the network, we apply repeated applications of the T1 move, a simple topological transition which changes the connectivity of a network edge.  A purely topological treatment of network statistics for the random T1 process was undertaken in the theoretical physics community \cite{aste1999glass,sherrington2002glassy} as a generalized version of the Ising model on planar graphs, and was found to exhibit glassy behavior under Metropolis-Hastings rules for rejecting certain configurations.  By allowing the spring network to equilibrate after each T1 transition, we are now able to study the geometric properties of the spring network, and for our purposes we are able to compute turning disorders.
  In Sec.\ \ref{subsec:rupture}, we randomly remove edges in a hexagonal lattice in a process topologically similar to edge rupturing in quasi-two-dimensional soap foams \cite{klobusicky2024planar}.  In this process, vertices are fixed and do not undergo equilibration after ruptures.  This produces irregular, nonconvex shapes which result in large turning disorders.   In Sec.\ \ref{sec:disc}, we discuss potential areas in which turning measures can be applied to the problem of classification of materials.

\section{A measure of disorder for microstructures} \label{sec:turndef}

A popular method for comparing shapes in two dimensions, proposed by Arkin et. al. \cite{arkin1989efficiently}, tracks the tangent vector of a shape as it traces its boundary.  A shape $A$ can be scaled so that its  perimeter is unit length, with its boundary traversed by a continuous curve $\varphi(s): [0,1] \rightarrow \mathbb{R}^2$ parametrized by arc length.   The \textit{turning function} $\Theta_A(s): [0,1] \rightarrow \mathbb{R}$ is the counterclockwise angle of the tangent vector $\varphi'(s)$ with respect to $x$-axis.  When the curve is only piecewise $C^1$, which occurs for polygons, we can make the turning function well-defined by requiring the function to be right continuous.  In what follows, the turning function will be extended to have a domain of $\mathbb R$ by imposing the identity
  \begin{equation}
  	\Theta_A(s  \pm  1) = \Theta_A(s) \pm 2 \pi. \label{periodic}
  \end{equation}
The floor value of $|s|$ is then interpreted as the number of complete traversals of the perimeter, while the sign of $s$ determines whether the traversal is clockwise or counterclockwise.

The \textit{p-turning distance} computes the $p$-norm between turning functions of two shapes, minimized over all possible rotations and starting points of the tangent vector.  For shapes $A$ and $B$, this takes the form
\begin{equation}
 d_p(A,B) = \left( \min_{\theta \in \mathbb R, t \in [0,1]} \int_0^1 |\Theta_A(s+t)- \Theta_B(s)+ \theta |^p ds\right) ^ \frac{1}{p}. \label{maindist}
\end{equation}
Note that we will also write $d_p(f,g)$ when comparing two turning functions $f,g \in L^p[0,1]$. The turning distance carries several desirable properties.  Chief among them is that for $p \ge 1$,  $d_p(\cdot, \cdot)$ is a well-defined metric among the space of simple closed curves, quotiented under translations, rotations, and scaling\footnote{In Lemma 1 of \cite{arkin1989efficiently}, a proof is given that $d_p$ is a metric for any $p>0$. However, the proof depends on the Minkowski inequality for $L^p$ spaces, which does not in general hold for $p \in (0,1)$.}. Thus, the metric is only concerned with the \textit{shape} of each curve.  It is also continuous with respect to perturbations of shapes.  In the Appendix, we show that continuity is locally Lipschitz for the case $p = 1$ for perturbing vertices, but that simple counterexamples exist for $p > 1$ in which Lipschitz continuity does not hold.

We now consider a measure of disorder for a microstructure contained in an open connected domain $\Omega \subset \mathbb R^2$ with unit area, partitioned as $\overline \Omega = \sqcup_{i = 1}^n \overline P_i$, where each $P_i$ is an open, connected domain whose boundary defines a simple closed curve.  The boundaries of the partition form a planar network $G$, where the side number $s(P_i)$ of each face is given by the number of vertices contained in its boundary. For such a network, we will define \textit{turning disorders} as an average of $d_2(P_i, A_i)$ for $i = 1, \dots, n$,  where $A_i$ is some predetermined ``ordered" shape.  Below, we will select three natural choices for $A_i$, and also consider weighting by cell area, which produces six measures for disorder.

One option for an ordered shape is to let $A_i = R_k$, where $R_k$ is a regular $k$-sided polygon.  For a polygon $P$ with $s(P)$ sides, we can compare against all $R_{k}$ for $k \ge 3$ and take a minimum of distances. However, a less computationally expensive method would compare $P$ to a regular polygon which also has $s(P)$ sides. In the same vein, we define the unweighted and weighted versions of the  \textit{regular turning disorder} as\begin{equation}
\mathcal D(G) = \frac{1}{n}\sum_{i=1}^{n}d_2(P_i,R_{s(P_i)}), \qquad \mathcal D_\mathrm{w}(G) = \sum_{i=1}^{n}|P_i|d_2(P_i,R_{s(P_i)}), \label{regdist}
\end{equation}
where $|P|$ denotes the area of a polygon $P$.  Note that because we assumed unit area for the network, the weighted disorder is indeed an area-weighted average of turning distances.

Trivalent networks are graphs that have degree three at each vertex, and are commonly found in two-dimensional microstructures such as cross sections of polycrystalline metals \cite{van2000grain} and quasi-two-dimensional foams  \cite{weaire1999physics}.  From Euler's theorem for planar networks, the average number of sides is approximately six for large networks.  Considering hexagonal lattices as the most ordered of trivalent networks, we can choose the regular hexagon $R_6$ as our comparison cell, and define unweighted and weighted \textit{hexagonal turning disorders} as

\begin{equation}
\mathcal D_6(G) = \frac{1}{n}\sum_{i=1}^{n}d_2(P_i,R_{6}), \qquad \mathcal D_\mathrm{6,w}(G) = \sum_{i=1}^{n}|P_i|d_2(P_i,R_{6}). \label{hexdist}
\end{equation}

Finally, we may be interested in comparing networks to microstructures such as wet foams whose areas are largely comprised of approximately circular cells \cite{furuta2016close}.  In this case, we define the \textit{circular turning disorder} as the average of turning distances between each cell and the circle $C$. The weighted and unweighted versions of this disorder are
\begin{equation}
 \mathcal{D}_{\mathrm{c}}(G) = \frac{1}{n}\sum_{i=1}^{n}d_2(P_i,C), \qquad \mathcal{D}_{\mathrm{c},\mathrm{w}}(G)= \sum_{i=1}^{n}|P_i|d_2(P_i,C). \label{circdist}
\end{equation}
As we will see in Sec. \ref{sec:compturn}, we have chosen the 2-distance $d_2$ in (\ref{regdist})-(\ref{circdist}) for reasons of computational efficiency.

There is no ordering among the disorders in (\ref{regdist})-(\ref{circdist}) that holds across all microstructures.  Indeed,  the turning distances of regular polygons and circles compared against even simple classes of shapes can produce interesting behavior.  For one example, shown in Fig. \ref{ar_plot}, we plot turning distances of rectangles $P_a$ with aspect ratio $a \in [1,20]$ against $R_4, R_6$, and $C$. There are multiple changes in the orderings of turning distances:
\begin{align*}
    &d_2(R_4, P_a)<d_2(C,P_a)<d_2(R_6, P_a) \qquad 1 \le a \le 1.51, \\
    &d_2(R_6, P_a)<d_2(C,P_a)<d_2(R_4, P_a) \qquad 1.52 < a < 2.80, \\
        &d_2(C,P_a)<d_2(R_6, P_a)<d_2(R_4, P_a) \qquad 2.81 < a < 6.20, \\
        &d_2(C,P_a)<d_2(R_4, P_a)<d_2(R_6, P_a) \qquad 6.21 < a < 8.46,\\
        &d_2(R_4, P_a)<d_2(C,P_a)<d_2(R_6, P_a) \qquad 8.47 < a < 13.20,\\
        &d_2(R_4, P_a)<d_2(R_6, P_a)<d_2(C,P_a) \qquad a > 13.21.\\
\end{align*}
The  plot for $d_2(R_6, P_a)$ also shows that changing aspect ratios does not ensure monotonic behavior in turning distances. In Sec. \ref{subsec:t1}, we find the nontrivial ordering of turning distances against aspect ratios manifesting in turning disorders of spring networks.

\begin{figure}
\centering
	\includegraphics[width = \linewidth]{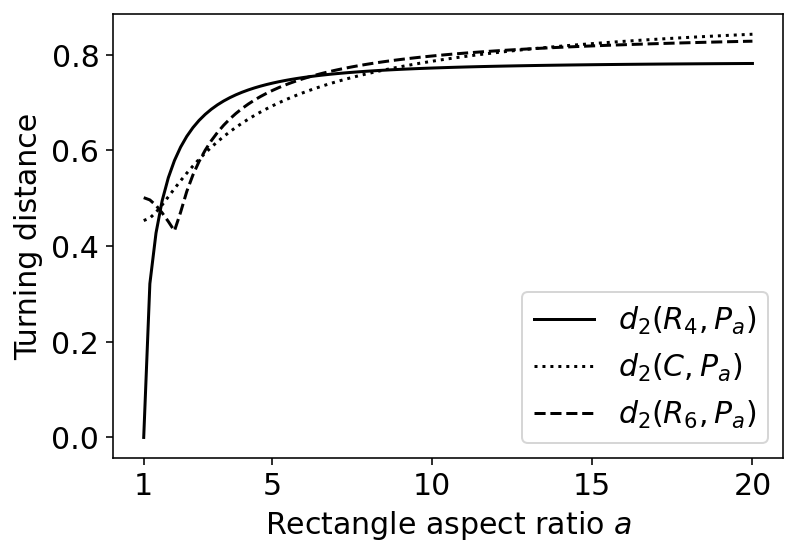}
	\caption{Turning distances of ordered shapes compared against rectangles $P_a$ with aspect ratios $1 \le a \le 20$} \label{ar_plot}
\end{figure}

It is straightforward to see that the turning function can be arbitrarily large for nonconvex shapes by constructing spiral-like curves with many rotations.  That the turning distance between two functions is unbounded is somewhat less transparent, as we have to consider minimum $L^p$ distances across all vertical and horizontal translations of turning functions.

\begin{proposition} \label{boundprop}
$d_p(\cdot, \cdot)$ is unbounded.	
\end{proposition}

\begin{proof}
For each $n >1 $, we can construct a sequence of spiral shapes with $n$ rotations with turning functions that approach, in $L^p(0,1)$, the piecewise-linear function
\begin{equation}
	g_n(s) = \begin{cases}
4\pi n s & s \in [0,1/2),\\
 4(1-n) \pi s+ (4n-3)\pi  & s \in [1/2, 1),\\
2\pi  & s = 1.
\end{cases}
\end{equation}
Specifically, consider a counterclockwise spiral tightly wound around a circle of radius $1/(4\pi n)$ beginning at the bottom of the circle with a tangent angle of 0. After the $n$th rotation, the curve sharply turns $\pi$ radians clockwise and spirals back to the origin point, where it turns counterclockwise by $\pi$ radians right before closing the curve. See Fig. \ref{spiral_plot}

\begin{figure}
\centering
	\includegraphics[width = \linewidth]{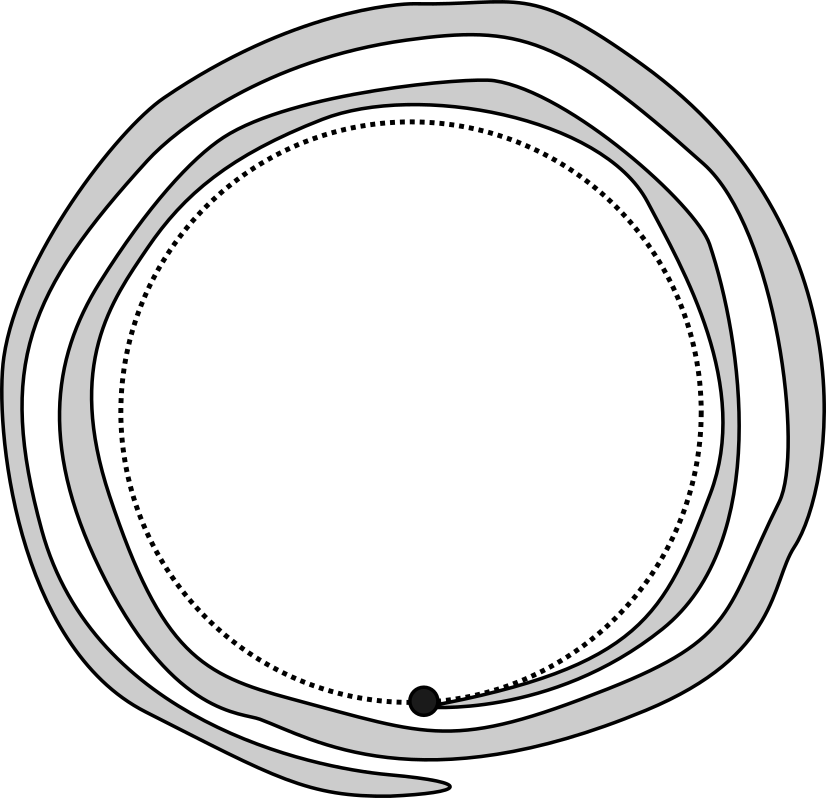}
	\caption{A spiral shape with two rotations wound around a circle (shown dashed) of radius $1/(4\pi n)$.  The turning function begins at the origin (black dot). The interior of the spiral shape is shaded in gray.} \label{spiral_plot}
\end{figure}

We now compare $g_n(s)$ against $h(s) = 2 \pi s$, the turning function for the unit perimeter circle.  Clearly, $h(s+t) = h(s)+2\pi t$, and so horizontal and vertical translations of $h$ are equivalent, meaning that we can set $t = 0$ in  (\ref{maindist}) and optimize $d_p(g_n,h)$ over $\theta$.  Furthermore, for all $\theta$ , the range $\{h(s) + \theta\}_{ s \in[ 0,1]}$ will be contained in an interval of length $2 \pi$.  With this fact, it follows that for all  $\theta \in \mathbb R$,  there is a set $A_\theta \subseteq [0,1]$ with $|A_\theta| \ge 1/2$ such that for all $s \in A_\theta$,
\begin{equation}
	|h(s) + \theta - g(s) | \ge n\pi/2.
\end{equation}
Then the  $p$-distance can be bounded from below by 
\begin{align}
 d_p(g,h)  = \left( \min_{\theta \in \mathbb R, t \in [0,1]} \int_0^1 |h(s+t)- g_n(s)+ \theta |^p ds\right) ^ \frac{1}{p} \\
 \ge \left(\min_{\theta \in \mathbb R }\int_{A_\theta} |h(s) + \theta  - g(s) | ^p ds \right)^{1/p} \ge 2^{-(p+1)/p}n\pi,
 \end{align}
 which diverges as $n\rightarrow \infty$.
\end{proof}

By comparing $g_n(s)$ against the turning function of a regular polygon, the proof of Prop. \ref{boundprop} follows similarly.  From this it is straightforward to show
\begin{cor}
The turning disorders in (\ref{regdist})-(\ref{circdist}) are all unbounded.
\end{cor}

Restricted to convex shapes, curvature cannot change signs, and so turning functions $f$ must be nondecreasing.  Paired with (\ref{periodic}), it follows that for all $s \in \mathbb R$ and $t \in [0,1]$
\begin{equation}
|f(t+s) - f(s)| \le 2\pi.
\end{equation}
This implies an upper bound of $d_p(f,g) \le 2 \pi$ between two turning functions for convex shapes. This bound is not sharp.  We conjecture that $\frac{\pi}2(p+1)^{-1/p}$ is the sharp upper bound between two closed convex curves for the $p$-distance.  This value is obtained as an increasing limit of turning distance between  the unit circle and  a sequence of rectangles $P_n$ with aspect ratios of $n$.  Stated formally, the sharp bound we are looking for would be $\sup_{f,g \in \mathcal H} d_p(f,g) $, where optimization is over the Helly space  $\mathcal H$ of nondecreasing functions $f: [0,1] \rightarrow [0, 2 \pi]$ under the endpoint constraints of $f(0) = 0, f(1) = 2 \pi$, and also the integral constraints
\begin{align} 
\int_0^1 \sin(f(s))ds = \int_0^1 \cos(f(s))ds = 0.
 \end{align}
We require $f \in \mathcal H$ to be nondecreasing for the shape to be convex, while the integral constraints ensure that the curve is closed.

\section{Computing the turning distance} \label{sec:compturn}

\subsection{General polygons}

In this subsection, we review the main results from Arkin et al.\@ \cite{arkin1989efficiently} for computing the turning distance.  Computing $d_p$ is most efficient for $p = 2$, since for each translation $t \in [0,1]$, we have an integral expression for the optimal rotation $\theta^*(t)$, thereby reducing the minimization problem (\ref{maindist}) to a single variable.  From Lemma 3 in \cite{arkin1989efficiently}, we find
\begin{proposition}
\label{l2lem}
Under the $L_2$ norm ($p = 2$), for polygons $A$ and $B$,  
\begin{equation}
d_2(A,B) = \left( \min_{t \in [0,1]} \int_0^1 (\Theta_A(s+t)- \Theta_B(s))^2ds -  \theta^*(t)^2\right)^{1/2},
\end{equation}
with the optimal rotation $\theta^*(t)$ given by 
\begin{equation}
\theta^*(t) = \int_0^1 g(s+t) - f(s)ds.\label{thetaopt}
\end{equation}
\end{proposition}

The optimization space for $t \in [0,1]$ can be further reduced to a discrete set when considering polygons, whose turning functions are step functions with jumps corresponding to corners.  For two polygonal turning functions $f$ and $g$, we define the set of \textit{critical events} $\mathrm{Cr}(f,g)$ as those values $t \in [0,1]$ for which $f(t+s)$ and $g(s)$ have at least one jump occurring at the same point.  Here we give a minor generalization of Lemma 4 and following discussion in \cite{arkin1989efficiently}.

\begin{proposition}
 \label{breaklem}
Let $p>0$. For polygons $A$ and $B$ with critical events $\mathrm{Cr}(A,B)$, 
\begin{equation}
d_p(A, B) = \left( \min_{\Theta \in \mathbb R, t \in \mathrm{Cr}(A,B)} \int_0^1 |\Theta_A(s+t)- \Theta_B(s)+ \theta |^p ds\right) ^ \frac{1}{p}.  \label{dpbreak}
\end{equation} 
\end{proposition}

\textit{Proof sketch:}
In \cite{arkin1989efficiently}, Prop. \ref{breaklem} is shown for $p = 2$, but the proof can be readily generalized for any $p>1$. The main point is that for any value of $\theta \in \mathbb R$ and $t \in [0,1]$,  the integral in (\ref{maindist}) is given by sums of rectangles whose widths are distances between discontinuity points in $\Theta_A(s+t)$ and $\Theta_B(s)$. As $t$ varies between points in $\mathrm{Cr}(A,B)$, rectangles grow or shrink at a linear rate, preventing a local minimum from occurring.  Thus, minima can only occur at values $t \in \mathrm{Cr}(A,B)$, where rectangles can be created or eliminated.

\qed

Arkin et al. \cite{arkin1989efficiently} then combine Props. \ref{l2lem} and \ref{breaklem} for computing the 2-turning distance.

\begin{cor} \label{breakcor}
For polygons $A$ and $B$, 
\begin{equation} d_2(A,B) = \left( \min_{ t \in \mathrm{Cr}(A,B)} \int_0^1 (\Theta_A(s+t)- \Theta_B(s))^2ds -  \theta^*(t)^2\right)^{1/2} . \label{d2break}
\end{equation} 
\end{cor}
If $A$ and $B$ have $m$ and $n$ vertices, respectively, it takes time $O(mn (m+n))$ to compute $d_2(A,B)$.  By considering a sorted list of critical events,  this can be refined to $O(mn \log(mn))$ (see Theorem 6 in \cite{arkin1989efficiently}).

\subsection{Regular polygons}

For turning distances between regular polygons, we can remove the optimization of critical events in (\ref{dpbreak}), and subsequently for $p = 2$ express distance in (\ref{d2break}) as a summation requiring no optimization.  Such a formula will serve as the foundation for deriving exact formulas for turning distances in Section \ref{sec:arch}.   In what follows, all polygons $P$ are scaled, translated, and rotated to have unit length, with a corner located at the origin, and an edge containing the origin lying on the positive x-axis. Furthermore, we will set the starting point of the turning function $\Theta_P(s)$ to be the origin, with the horizontal side to its right as the first edge traversed as $s$ increases, meaning $\theta_{P}(s) =  0$ for $s\in [0, s_1)$, where $s_1$ is the length of the side on the $x$-axis.  Under these conventions, the turning function $f_n(s)$ for a regular polygon $R_n$ with $n$ sides is the step function
\begin{equation}
	f_n(s) = \sum_{i  \in \mathbb Z} (2\pi i/n) \mathbf 1(s \in [i/n, (i+1)/n)). \label{turnregdef}
\end{equation}
Here, we use the indicator $\mathbf 1(A)$ to equal 1 if $A$ is true and 0 if $A$ is false.

For regular polygons, we use the symmetry of the shape to show the invariance of critical events in the turning distance.  In terms of turning functions, this manifests in the equivalence of horizontal shifts and vertical shifts.  Namely, for all $s \in \mathbb R$ and $ i \in \mathbb Z$, we have the identity
      \begin{equation}
 f_n(s+ i/n) = f_n(s) + 2\pi i/n.	 \label{shiftisrot}
 \end{equation}
 This identity could be used to remove optimization over critical events for the turning distance between regular polygons.

\begin{proposition} \label{genppoly}
Let $f_n$ be defined as in (\ref{turnregdef}).  Then for $n,k \ge 2$, 
\begin{equation}
 d_p(R_n,R_k) = \left( \min_{\theta \in \mathbb R} \int_0^1 |f_n(s)- f_k(s)+ \theta |^p ds\right) ^ \frac{1}{p} \label{pnot}
\end{equation}
and 
\begin{equation}
 d_2(R_n,R_k) = \left( \int_0^1 (f_n(s)- f_k(s))^2 ds-\theta^*(0)^2\right) ^ \frac{1}{2}.\label{pnot2}
\end{equation}

 \end{proposition}

 \begin{proof}
  From Prop. \ref{breaklem}, we find that maxes and mins for the $p$-turning distance occur at critical events.  For regular polygons $R_n$ and $R_k$, the critical events are
  \begin{equation}
  \mathrm{Cr}(R_n,R_k) =  \{l/n+ m/k \in [0,1]: l = 0, \dots, n-1, \quad m = 0, \dots, k-1\}. \label{ngonevents}
  \end{equation}

Optimization over $\theta$ over these values of $t \in   \mathrm{Cr}(R_n,R_k)$ is equivalent to an optimization at $t = 0$, as for any $m/k +l/n \in \mathrm{Cr}(R_n,R_k)$, 
  \begin{align*}
& \min_{\theta \in \mathbb R}\int_0^1 	|f_n(m/k +l/n+ s)- f_k(s)+\theta|^pds \\&= \min_{\theta \in \mathbb R}\int_0^1 	|f_n(m/k + s)- f_k(s) + 2 \pi l/n  +\theta|^p ds \\
&=  \min_{\theta \in \mathbb R}\int_{m/k}^{1+m/k} |f_n(s)- f_k(s-m/k) + 2 \pi l/n +\theta|^p ds \\
&= \min_{\theta \in \mathbb R}\int_{m/k}^{1+m/k} |f_n(s)- f_k(s) + 2 \pi (l/n- m/k) +\theta|^p ds \\
 &= \min_{\theta \in \mathbb R}\int_{0}^{1} |f_n(s)- f_k(s) + 2 \pi (l/n- m/k) +\theta|^p ds \\
 &= \min_{\theta \in \mathbb R}\int_{0}^{1} |f_n(s)- f_k(s)+\theta|^p ds.
 \end{align*}
The second to last equality uses (\ref{periodic}), which implies $f_n-f_k$ is $1$-periodic.  The last equality holds since the shifted variable  $\theta' = \theta+ 2 \pi (l/n- m/k)$  is still optimized over $\mathbb R$. Then (\ref{pnot}) follows immediately from Prop. \ref{breaklem}, and (\ref{pnot2}) follows from Prop. \ref{l2lem}  and Corollary \ref{breakcor}.  
	\end{proof}

    \begin{remark}In this result, we have allowed for the ``2-gon" $R_2$. While technically not a polygon, $R_2$ is interpreted as a line segment connecting $(0,0)$ and $(1/2,0)$, with a turning function of 
\begin{equation}
f_2(s) =  \begin{cases}
        0 & s \in [0,1/2), \\
        \pi & s \in [1/2, 1).
   \end{cases} \label{r2turn}
\end{equation}
Geometrically, this can be seen as a limiting shape of rectangles with aspect ratio $a \rightarrow \infty$, and will be used for approximating needlelike regions in Sec. \ref{subsec:t1}.
\end{remark}

The 2-turning distance between regular polygons, written as an integral in (\ref{pnot2}), can instead be written as a sum using floor functions through a direct calculation.  For simplicity in presentation, in what follows we will write $D_2(P_1, P_2) = (d_2(P_1, P_2)/\pi)^2$.

\begin{proposition}
Let $n,k \ge 2$.  
\begin{enumerate}
\item The 2-turning distance between $R_n$ and $R_k$ satisfies \begin{equation}
	D_2(R_n, R_k) = - \left(\frac 1n - \frac 1k \right)^2+ \frac{4}{nk}\sum_{i = 0}^{nk}	\left (\frac{1}{k}\left\lfloor \frac in\right \rfloor - \frac 1n\left\lfloor \frac ik\right \rfloor \right )^2. \label{mainregpoly}
\end{equation}
\item  $D_2(R_n, R_k)$ can be computed in $O((n+k)\log(n+k))$ time.
\end{enumerate}

 \end{proposition}

\begin{proof}

To show (1), we note that from Prop. \ref{genppoly}, the first term in (\ref{mainregpoly}) is the optimal rotation at $t = 0$, 
\begin{equation}
\theta^*(0)	= \pi\left(\frac 1n - \frac 1k \right).\label{rot0reg}
\end{equation}
This immediately follows from (\ref{thetaopt}) and noting that 
\begin{equation}
\int_0^1 f_n(s)ds = \pi\left(1-\frac{1}{n}\right). 
\end{equation}
For the summation term, we partition $[0,1]$ into $nk$ equally sized intervals. In the interval $s \in [i/nk, (i+1)/nk))$, $f_n(s) \equiv \frac{2 \pi}{k}\left\lfloor \frac in\right \rfloor $ and $f_k(s) \equiv \frac{2 \pi}{n}\left\lfloor \frac ik\right \rfloor $.  The result then follows from (\ref{pnot2}).

To show (2), we can compute (\ref{mainregpoly}) as follows:
\begin{enumerate}
	\item Create a sorted list $c_1 \le \dots \le  c_{n+k}$ of the jump discontinuities $\{l/n\}_{l = 0, \dots n-1} \cup\{l/k\}_{l = 0, \dots k-1}$ of the turning functions.
	\item  Compute lengths $\ell_j = c_{j+1}- c_{j} $.
	\item Compute the sum
	\begin{equation*}
		S = \sum_{j= 1}^{n+k} \ell_j (f_n(c_j)-f_k(c_j))^2.
	\end{equation*}
\end{enumerate}
Sorting takes $O((n+k)\log(n+k))$ time, and steps 2 and 3 each take $O(n+k)$ time. 
\end{proof}

\section{Closed form expressions for special cases} \label{sec:compspec}

The distribution of lengths between critical events of $R_n$ and $R_k$ depends upon the prime factorizations of $n$ and $k$ and their common divisors.  In what follows, we'll give some simplifications for special cases of $n$ and $k$ in which we can give closed formulas for $d_2(R_n,R_k)$. We begin with cases where $n$ and $k$ share common factors.

\begin{proposition}
For positive integers $n, k, a \ge 2$,
\begin{equation}
d(R_{an}, R_n) =  \frac{\pi}{an} \sqrt{\frac{a^2-1}3} \label{atimesn}
\end{equation}
and 
\begin{equation}
d(R_{an}, R_{ak}) = \frac{1}{a}d(R_{n}, R_{k}). \label{totaldiv}
\end{equation} 
\end{proposition}

\begin{proof}

From (\ref{thetaopt}),
\begin{align*}
    &\theta^*(t) = \frac{\pi}n \left( \frac{a-1}{a}\right).
\end{align*}
Next, we observe that the function $f_{an}(s)-f_{n}(s)$ is $1/n$ periodic, with explicit step function form
\begin{equation}
f_{an}(s)-f_{n}(s) = 	 \sum_{i  \in \mathbb Z} \frac{2\pi (i \hbox{ mod } a)}{an} \mathbf 1(s \in [i/an, (i+1)/an)).
\end{equation}  
This yields
\begin{align*}
    &\int_0^1 (f_{an}(s)- f_{n}(s))^2ds = n \sum_{i=1}^{a-1} \left( \frac{1}{an}\right) \left( \frac{2\pi i}{an} \right) ^2 \\
    &=\frac{4\pi ^2}{a^3n^2} \sum_{i=1}^{a-1} i^2 = \frac{2\pi^2(a-1)(2a-1)}{3a^2n^2}.
\end{align*}
Substituting into (\ref{pnot2}) establishes (\ref{atimesn}), since
\begin{align*}
   d_2(R_n, R_k) &=  \left( \frac{2\pi^2(a-1)(2a-1)}{3a^2n^2}-  \left(\frac{\pi}n \left( \frac{a-1}{a}\right)\right) ^2 \right) ^ \frac{1}{2} \\
    &= \frac{\pi}{an} \sqrt{\frac{a^2-1}3}.
\end{align*}

To show  (\ref{totaldiv}), we note that multiplying the number of sides of a regular polygon is equivalent to a rescaling of the turning function.  In other terms, for all $s \in \mathbb R$, 
\begin{equation}
f_{an}(s) = \frac{1}{a} f_n(as).
\end{equation}
It follows that $f_{an} - f_{ak}$ is $1/a$ periodic.
We now can show that, for $p >0 $,
\begin{align}
\int_0^1 (f_{an}(s) - f_{ak}(s))^pds &= \sum_{i = 1}^a \int_{(i-1)/a}^{i/a} (f_{an}(s) - f_{ak}(s))^pds \\
&= \frac{1}{a^{p-1}}\int_0^{1/a} (f_{n}(as) - f_{k}(as))^pds \\
 &= \frac{1}{a^p}\int_0^1 (f_n(s)-f_k(s))^pds.
\end{align}
The above can then be inserted into (\ref{d2break}) with $p = 1,2$ to obtain the result.  \end{proof}

\begin{cor} \label{gcdcor}
For $q = \gcd(n,k)$,
\begin{equation}
d_2(R_{n}, R_{k}) = \frac{1}{q}d_2(R_{n/q}, R_{k/q}).  \label{gcdfact}
\end{equation}
\end{cor}

If $n$ and $k$ are relatively prime, then typically the gaps between critical events are erratic, and the summation formula (\ref{mainregpoly}) will be required to find the turning distance.  However, we mention one case in which a closed form solution exists between $n$- and $(n+1)$-gons.  Note that we write $f(n)\sim g(n)$ to mean $\lim_{n\rightarrow \infty} f(n)/g(n) = 1$.

\begin{proposition}
	For $n \ge 2$, the turning distance between $R_n$ and $R_{n+1}$ satisfies
\begin{equation}
D_2(R_{n}, R_{n+1})	 = \frac{2n^2+2n-1}{3n^2(n+1)^2},
\end{equation}
and thus $d_2(R_{n}, R_{n+1}) \sim \frac {\sqrt{6}\pi}{3n} .$
\end{proposition}

\begin{proof}

From (\ref{thetaopt}),  the optimal rotation between $R_n$ and $R_{n+1}$ is 
\begin{equation}
	\theta^*(0) = \frac{\pi}{n(n+1)}.
\end{equation}
For the computation involving $(f_n-f_{n+1})^2$, consider a partition of $[0,1]$, with $n$ intervals $I_i = [(i-1)/n, i/n)$ for $i = 1, \dots, n$. For $s \in I_i$, $f_n(s) \equiv \pi(i-1)/n$ .  However, $f_{n+1}$ assumes two values in each $I_i$.  For $I_i^{(1)} :=  [(i-1)/n, i/(n+1))$ and $I_i^{(2)} := [i/(n+1), i/n)$, 
\begin{equation}
	f_{n+1}(s) = \begin{cases}
\pi(i-1)/(n+1) & s \in I_i^{(1)},\\
 \pi i/(n+1)  & s \in  I_i^{(2)}.
\end{cases}
\end{equation}
Using closed-form expressions for the sums of integers and squares,  we now compute the turning distance as
\begin{align*}
&D_2(R_n, R_{n+1}) = \frac{1}{\pi^2}\left(\sum_{i = 1}^n\int_{I_i^{(1)}\cup I_i^{(2)}}(f_n(s)-f_{n+1}(s))^2ds - \theta^*(0)^2\right)\\ 
&= 4\sum_{i = 1}^{n} \left[\frac{n+1-i}{n(n+1)}\left( \frac{i-1}n- \frac{i-1}{n+1} \right) ^2+ \frac{i}{n(n+1)}\left( \frac{i-1}n- \frac{i}{n+1} \right)^2\right]-\frac{1}{n^2(n+1)^2} \\
&=\frac{4}{n^3(n+1)^3} \sum_{i=1}^{n} \left( i^2(1-n) + i(n^2-2) + (n+1) \right) -\frac{1}{n^2(n+1)^2} \\
&= \frac{2n^2+2n-1}{3n^2(n+1)^2}.
\end{align*}
\end{proof}

We can also compute turning distances between a general polygon and a circle.  
The unit circle $C$, positioned so that its lowest point is placed at the origin, has the turning function of $\Theta_{C}(s) = 	2\pi s$. For an $n$-sided polygon, denote side lengths, beginning from the segment on the $x$-axis, as $s_0, s_1, \dots, s_n$, where $s_0 = 0$. Denote cumulative perimeters as $x_i = \sum_{k = 0}^i s_k$.  Finally, denote corner turning angles $\theta_i$ as the right-sided limit of the turning function at the $ith$ corner traversed, with $\theta_0 = 0$ and $\theta_n = 2 \pi $. For a polygon $P$ with $n\ge 2$ sides, we then have the following simplifications for $d_2(C,P)$.

\begin{proposition}
Let $g(s) = 2\pi s$ be turning function for the circle $C$, and  let $h(s)$ be the turning function for an $n$-sided polygon $P$ with cumulative perimeters $x_0, \dots, x_n$ and corner turning angles $\theta_0, \dots, \theta_{n-1}$. Then 
\begin{align}
	D_2(C,P) = \frac{1}{6 } \sum_{i = 1}^n \left[(2 x_i- \theta_{i-1}/\pi)^3 - (2 x_{i-1}- \theta_{i-1}/\pi)^3\right] - \left(1-  \sum_{i=1}^{n} \theta_{i-1}(x_{i} - x_{i-1})/\pi\right)^2 \label{circturn}
\end{align}
	and
\begin{equation}
	d_2(C, R_n) = \frac{\sqrt{3}\pi}{3n}. \label{circreg}
\end{equation}
\end{proposition}

\begin{proof} 
By taking the $L_2$ limit of $f_k \rightarrow g(s) = 2 \pi s$, we find that (\ref{d2break}) also holds for $d_2(f_n, h(s))$. This convergence follows the dominated convergence theorem with the constant uniform bound $|f_k-h(s)| \le 2\pi+\max_{s \in [0,1]}{h(s)}$.  Furthermore, since $g(s+t) = g(s)+ 2\pi t $ (a limiting version of (\ref{shiftisrot})), we can repeat the proof for Prop. \ref{genppoly} and set $t = 0$ for the optimal shifting, meaning (\ref{pnot2}) also holds with $d_2(C,P)$. 

We can now compute (\ref{d2break}) by breaking up the integral for each side of $P$. We find 
\begin{align}
\int_0^1 (g(s) - h(s))^2ds &= 
\sum_{i = 1}^n \int_{x_{i-1}}^{x_i} (2 \pi s-\theta_{i-1})^2ds \\&=   \frac{\pi^2}{6} \sum_{i = 1}^n (2 x_i- \theta_{i-1}/\pi)^3 - (2 x_{i-1}- \theta_{i-1}/\pi)^3.
\end{align}
Similarly, for the optimal rotation term, 
\begin{align}
\theta^*(0) &= \int_0^1 (g(s) - h(s))ds \\&= \int_0^1 2\pi sds - \sum_{i = 1}^n \int_{x_{i-1} }^{x_{i}}	\theta_{i-1}ds \\&= \pi\left(1 - \sum_{i=1}^{n} \theta_{i-1}(x_{i} - x_{i-1})/\pi\right).
\end{align}
To show (\ref{circreg}), we can either substitute  $x_i = i/n$ and $\theta_i = 2 \pi i/n$ into (\ref{circturn}), or simply take the $L_2$ limit of (\ref{atimesn}) as   $a \rightarrow \infty$.  This limit exists by again applying the dominated convergence theorem, as $f_{an}(s)-f_n(s) \rightarrow  g(s)-f_n(s)$ pointwise, and $|f_{an}-f_n|$ are uniformly bounded by $4 \pi $.
\end{proof}

\section{Archimedean Lattices} \label{sec:arch}

For an example of ordered microstructure, we consider Archimedean lattices, a class of periodic tilings which are vertex-transitive, meaning that between any two vertices $u$ and $v$, there exists a graph isomorphism that maps $u$ to $v$.  There are 11 such lattices \cite[Ch. 2]{grunbaum1987tilings}, each of which are comprised of regular polygons. These lattices serve as a nontrivial class of microstructures for which we can compute exact values of turning disorder.  They also demonstrate ways in which the disorders differ.  In particular, since tiles are regular polygons, regular turning disorders are identically zero. However, with the exception of the hexagonal (honeycomb) lattice having zero hexagonal turning disorder,  Archimedean lattices have  nonzero values for both circular and hexagonal disorders.  
The three lattices we will consider are the $(3,12^2), (4,8^2)$, and $(4,6, 12)$ lattices, shown in Fig. \ref{archfig}.  These are the three Archimedean lattices which have more than one tile and are trivalent, or having all vertices of degree 3.  

For a lattice $A$ tiling the plane, we understand the various disorders as the limit of disorders restricted to the square $S_n = [-n,n] \times [-n,n]$ as $n\rightarrow \infty$, written as
\begin{equation}
	\overline{\mathcal D}(A) = \lim_{n\rightarrow \infty} \mathcal D(S_n \cap A).
\end{equation}
From this definition, we immediately obtain 
\begin{equation}
	\overline{\mathcal D}(A) = \overline{\mathcal D}_\mathrm{w}(A) = 0
\end{equation}
for all Archimedean lattices.  This holds because all interior tiles of the lattice in $S_n$ are regular polygons whose corresponding terms in (\ref{regdist}) are zero. These interior tiles comprise an arbitrarily large proportion of the lattice as $n\rightarrow \infty$.

For other disorders, this limit is well-defined for Archimedean lattices, as each lattice can be tiled with a single \textit{fundamental region}, defined as a finite subregion of the plane containing fractions of tiles which generate the tiling when symmetries (rotations and translations) are applied \cite{fundper}.  Archimedean tilings are periodic, and thus have at least one fundamental region.  In a fundamental region, denote area weighted and unweighted proportions of regular $n$-gons as $p_n$ and $q_n$, respectively.   Although fundamental regions may not be unique, it still holds that

\begin{proposition}
Any fundamental region for an Archimedean lattice contains the same proportions $p_n$ and $q_n$ for $n \ge 3$.
\end{proposition}

\begin{proof}
For any fundamental region $F \subset \mathbb R^2$, as $n \rightarrow \infty$ the interior of the square $S_n$ will contain rotated and translated copies of $F$.  Along the boundary of $S_n$ there may  be incomplete pieces of $F$, but these regions comprise an arbitrarily small proportion of all tiles as $n\rightarrow \infty$ when considering weighting by either total polygons or  area.  Each of these full copies of $F$ in $S_n$ contain the same fractions of regular polygons, so as $n\rightarrow \infty$, the proportion of regular $n$-gons in $S_n$ approach those of $F$.  The same argument can be made for a different fundamental region $\tilde F$.  By uniqueness of limits, $\tilde F$ must have the same $n$-gon proportions as $F$.  
\end{proof}
    Fundamental regions are shown in Fig. \ref{archfig}, shaded in gray. For the $(4,8^2)$ lattice, a fundamental region contains the area of 1/2 of a square and 1/2 of an octagon. the $(4,6,12)$ lattice, the fundamental cell contains the area of 1 hexagon, 3/2 of a square, and 1/2 of a  dodecagon.  For the $(3,12^2)$ lattice, the fundamental cell contains the area of 1 triangle and 1/2 of a  dodecagon.

\begin{figure}
\centering
	\includegraphics[width = \linewidth]{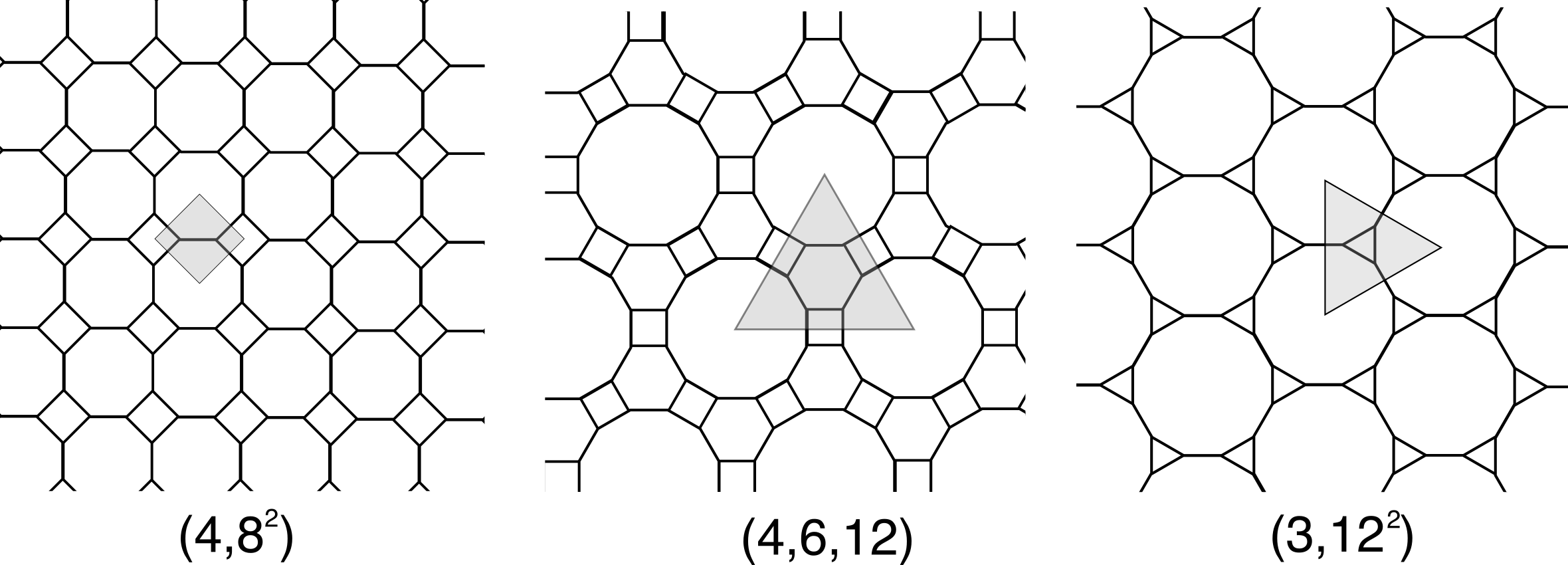}
	\caption{Three examples of Archimedean lattices, with a fundamental region for each lattice shaded in gray.} \label{archfig}
\end{figure}

Turning disorders can then be computed as weighted sums with respect to contributions of regular polygons in fundamental regions. Exact expressions can be found using (\ref{mainregpoly}) with the Python package \texttt{sympy} \cite{meurer2017sympy} for simplifying fractions and roots. We comment that for the hexagonal lattice $A(6^3)$, which is also Archimedean and trivalent, all turning disorders are equivalent to computing $d_2(R_6, P)$ for some ordered shape $P$.  In particular, for hexagonal and regular disorders, the disorder is simply $d_2(R_6, R_6) = 0$.  The circular disorder, from (\ref{circreg}), is $d_2(R_6, C)  = \frac{\sqrt 3\pi}{18} = 0.3023$.

For unweighted turning disorders, we use unweighted probabilities $q_k$ of fundamental regions, which are defined as the fraction of regular $k$-gons in the regions divided by the total fraction of regular polygons  which comprise the fundamental region. These are 
\begin{align*}
q_3 &= 2/3, \quad q_{12} = 1/3,  &(\mathbf{3,12^2}) \\
q_4 &= 1/2, \quad  q_8  = 1/2,  &(\mathbf{4, 8^2})\\
q_4 &= 1/2, \quad q_6 = 1/3, \quad q_{12} = 1/6.  &(\mathbf{4,6, 12})
\end{align*}
Unweighted turning disorders can then be computed as a weighted average of turning distances under the various $q_k$'s.  For instance, using (\ref{atimesn}), the unweighted hexagonal disorder for $A(3,12^2)$ can be quickly calculated as
\begin{align}
\overline{\mathcal{D}}_6(A(3,12^2)) &= q_3 d_2(R_3, R_6) + q_{12} d_2(R_{12},R_6) \\ &= \frac{2}{3} \cdot \frac{\pi}{6} + \frac 13 \cdot \frac{\pi}{12} = \frac{5\pi}{36} = 0.4363.
\end{align}

Let $p_k$ be the probability that a randomly chosen point in a fundamental region belongs to a regular $k$-gon.  These can be computed as fractions involving areas $A_k$ of a regular $k$-gon with a side length of 1.  In particular, we use $A_3 = \sqrt 3/4, A_4 = 1, A_6 = 3\sqrt 3/2, A_8 = 2 + 2\sqrt 2$ and $A_{12} = 6 + 3 \sqrt 3$. Area-weighted probabilities are then calculated as
\begin{align*}
&p_3 = 7-4\sqrt 3 = 0.0718, \quad p_{12} = 4\sqrt 3-6 = 0.9282,  &(\mathbf{3,12^2})\\
&p_4 = 3- 2\sqrt 2 = 0.1716, \quad p_8 = 2\sqrt 2 - 2= 0.8284,  &(\mathbf{4, 8^2}) \\
&p_4 = \frac{2\sqrt 3- 3}3 = 0.1547, \quad p_6 = 2-\sqrt 3= 0.2679, \quad p_{12} = \frac{\sqrt 3}{3}= 0.5774.  &(\mathbf{4,6,12})
\end{align*}
Weighted turning disorders can now be computed using $p_k$.  For instance, the weighted hexagonal disorder is given by
\begin{align}
\overline{\mathcal{D}}_{6,\mathrm{w}}(A(3,12^2)) &= p_3 d_2(R_3, R_6) + p_{12} d_2(R_{12},R_6)\\ &= (7-4\sqrt 3)\cdot \frac{\pi}{6} + (4\sqrt 3-6)\cdot \frac{\pi}{12} = \frac{\pi}3(2-\sqrt 3) =  0.2806.   
\end{align}

\begin{table}
  \begin{tabular}{|c||c |c|c|} \hline
& \centering$(4,8^2)$ & \centering $(3,12^2)$ &  $(4,6,12)$ \\ \hline
   & \multicolumn{3}{c|}{\textbf{Exact expression} } \\ \hline
 $\overline{\mathcal{D}}_{6}$ & {$\frac{\pi}{144}(2\sqrt{33}+\sqrt{69}) $} & $\frac{5\pi}{36}$ & $\frac{\pi}{72}(\sqrt{33}+1)$  \\ \hline
$\overline{\mathcal{D}}_{6,\mathrm w}$ & $\frac{\pi}{36}(3\sqrt{33}+\sqrt{138}-2\sqrt{66}-\sqrt{69}) $ & $\frac{\pi}{3}(2-\sqrt 3) $  & $\frac{\pi}{36}(2\sqrt{11}+\sqrt{3}-\sqrt{33})$\\ \hline
 $\overline{\mathcal{D}}_{\mathrm{c}}$ & $\frac{\sqrt 3 \pi}{16}$ & $\frac{\sqrt 3 \pi}{12}$ & $\frac{7\sqrt 3 \pi}{108}$  \\ \hline
$\overline{\mathcal{D}}_{\mathrm{c},\mathrm{w}}$ & $\frac{\pi}{12}(2\sqrt{3}- \sqrt{6})$ & $\frac{ \pi}{18}(11\sqrt{3}- 18) $ & $\frac{\pi}{36}(1+\sqrt 3) $ \\ \hline 
   & \multicolumn{3}{c|}{\textbf{Decimal representation}} \\ \hline
 $\overline{\mathcal{D}}_{6}$ & $0.4319$ & $ 0.4363$ & $ 0.2942$  \\ \hline
$\overline{\mathcal{D}}_{6,\mathrm w}$ & $ 0.3863$ & $  0.2806$  & $ 0.2287$\\ \hline
 $\overline{\mathcal{D}}_{\mathrm{c}}$ & $ 0.3401$ & $0.4534$ & $0.3527 $  \\ \hline
$\overline{\mathcal{D}}_{\mathrm{c},\mathrm{w}}$ & $ 0.2656$ & $ 0.1837$ & $ 0.2384$ \\ \hline
 \end{tabular}
 \caption{Turning disorders for Archimedean lattices.} \label{archtable}
 \end{table}

In Table \ref{archtable}, we display weighted and unweighted turning distances for each of the lattices.  There are several observations which show how the different disorders can capture different features of the lattices:
\begin{itemize}
\item First, we note that the $(3,12^2)$ lattice produces the smallest disorder, with $\overline{\mathcal{D}}_{\mathrm{c}, \mathrm{w}} = 0.1837$.  Here, 12-gons take up $p_{12} = 0.9282$ of the lattice area, and produce relatively small distances $d_2(R_{12}, C)= \sqrt{3}\pi/36=  0.1511$.  Removing area weighting, however, produces the largest disorder of the lattices, with $\overline{\mathcal{D}}_{\mathrm{c}} = 0.4534$, largely due to the relatively larger distance of $d_2(R_3, C) = \sqrt{3}\pi/9 = 0.6046$.

\item While the $(4,6,12)$ lattice has three types of tiles, it gives the lowest value of hexagonal disorders.  This is largely due to the presence of hexagons which have zero hexagonal distance. 

\item Both squares and octagons are ``two sides away" from the hexagon, but we find that $R_8$ is closer to $R_6$ than $R_4$ for the turning distance, with $d(R_4, R_6) = \sqrt{33}\pi/36 = 0.5013$ and $d(R_8, R_6) = \sqrt{69}\pi/36 = 0.3624$.  Therefore, since weighting by areas gives more bias to octagons, $\overline{\mathcal{D}}_{6,\mathrm w}(A(4,8^2)) < \overline{\mathcal{D}}_{6}(A(4,8^2))$.  
 
\item In the examples shown thus far, weighting by area reduces disorder for both hexagonal and circular disorder, but as we will see for the rupturing process in Sec. \ref{subsec:rupture}, this does not occur in general. 

\end{itemize}

\section{Examples of disordered microstructure} \label{sec:dis}

In what follows, we will study two stochastic processes on trivalent networks.  In both of these processes, typical cells are not regular polygons or circles, and so the general algorithm for the turning distance needs to be applied to calculate the hexagonal and regular disorders.  The circular disorder is computed using (\ref{circturn}).

\subsection{Spring networks with T1 moves} \label{subsec:t1}

\begin{figure}
\centering
	\includegraphics[width = \linewidth]{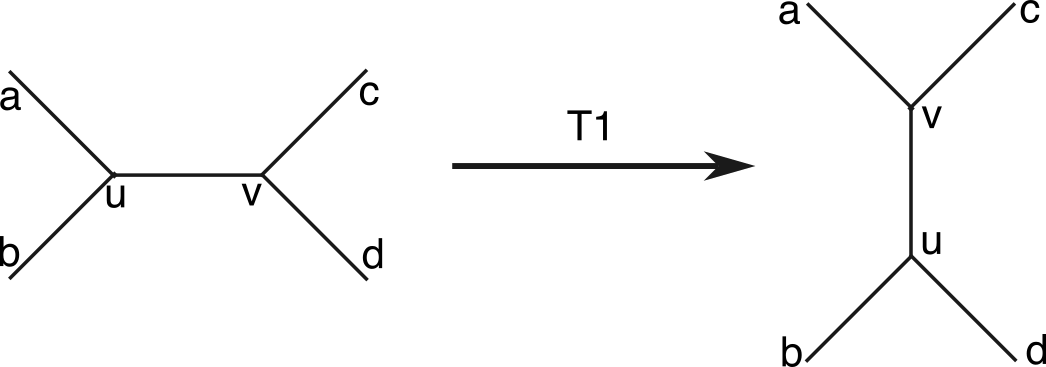}
	\caption{A T1 move.} \label{t1_move}
\end{figure}

For our first example of a stochastic process on networks with increasing disorder, we consider a hexagonal lattice whose edges undergo a series of T1 moves followed by repositioning of vertices from Tutte's method of drawing a planar graph.  See Fig.  \ref{t1_move} for a visualization of a single T1 move. In this figure, we consider an edge $(u,v)$ and its vertex neighbors $\{a,b,c,d\}$.  To perform a T1 move on the edge $(u,v)$ we change the connectivity of  $u$ to now be adjacent to vertices $b$ and $d$,  and $v$ to be adjacent to $a$ and $c$.   Denoting $C_n$ for a cell with $n$ sides, four cells change their number of sides from a T1 move through the reaction
\begin{equation}
	C_i + C_j + C_k + C_l \rightharpoonup C_{i +1} +C_{j+1}+ C_{k-1} + C_{l-1}. \label{reactt1}
\end{equation}
For our simulations, we forbid reactions to occur which result in cells with fewer than three sides. This is done to keep the graph simple (i.e., having at most a single edge between two vertices and no loops).

The reaction  maintains planarity of the network and can be described completely through adjacency matrices.  By selecting random edges with uniform probability, we can define a discrete-time stochastic process on the network which undergoes a large number of T1 moves.   When trying to represent multiple T1 processes as a geometric graph,  it becomes readily apparent that inserting edges in an ad-hoc manner quickly becomes unmanageable.  A natural way to address the problem of vertex assignment is to employ the well-known Tutte algorithm for drawing a planar graph \cite{tutte1963draw}.  Given a collection of locations for pinned vertices on the boundary of a convex set in the plane and an adjacency matrix for a 3-regular graph, Tutte's algorithm gives a planar representation of the network in which all vertices are contained in the convex set, and edges between vertices are drawn as line segments.  This is achieved by interpreting edges as linear springs and solving the associated force balance equations which form a linear system. Specifically, to compute the $n\times 2$ matrix $P_2$ of the interior vertices, given the pinned $m\times 2$ matrix of boundary vertices $P_1$, along with the adjacency matrix $A$ and diagonal degree matrix $D$ of the network, we compute (see \cite{tuttecmu})
\begin{equation}
P_2 = -L_2^{-1}BP_1,
\end{equation}
where $B$ is the $n \times m$ matrix of $-A$ restricted to connections between interior and boundary vertices, and $L_2$ is the $n \times n$ matrix $D-A$ restricted to interior vertices.

\begin{figure}
\centering
	\includegraphics[width = \linewidth]{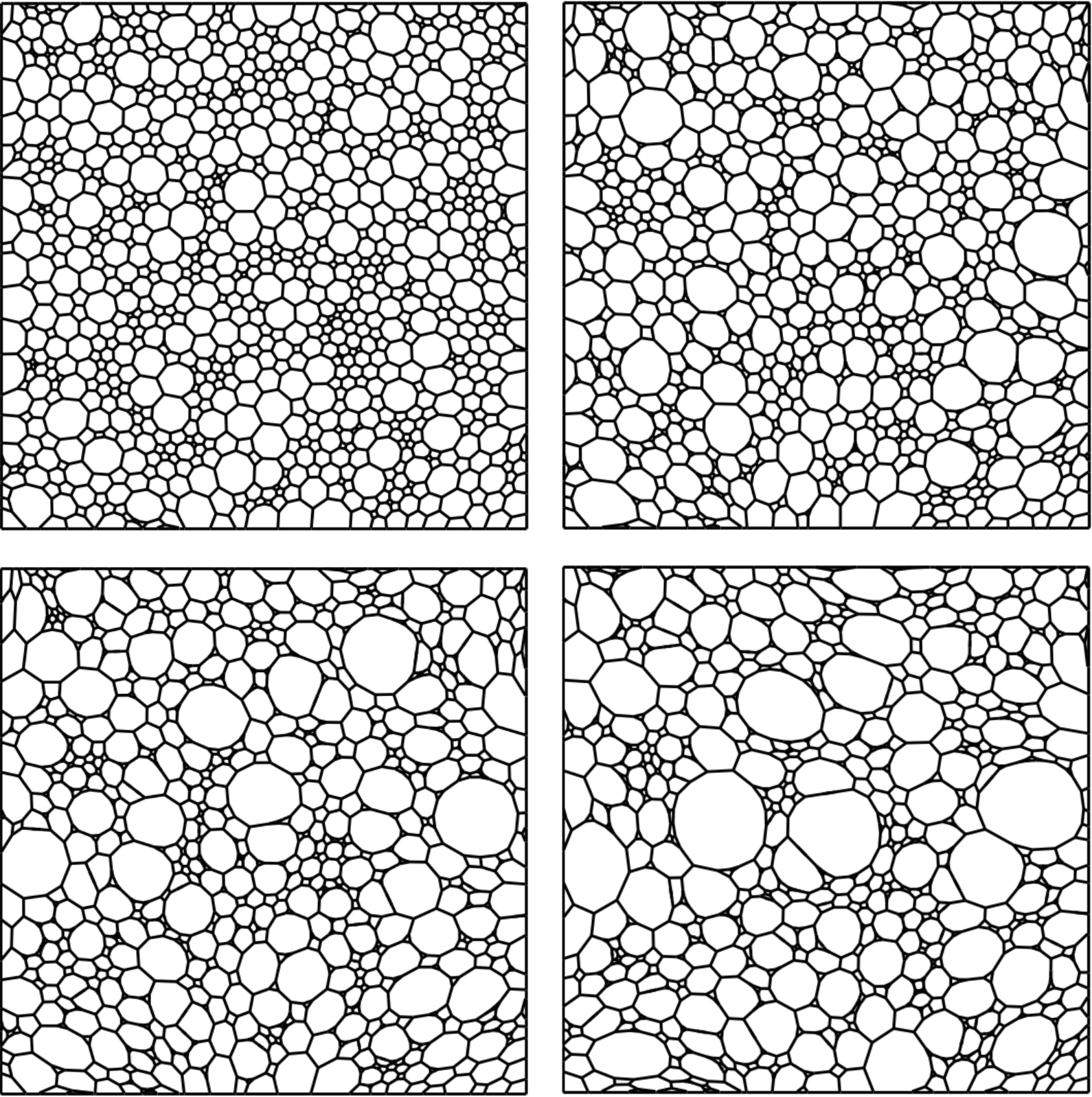}
	\caption{Snapshots of the T1 process with initial conditions of a Voronoi diagram with 1000 cells after the Tutte algorithm. From top left to bottom right: the process after 0, 1000, 2000, and 3000 T1 moves.} \label{t1sim}
\end{figure}

For simulating T1 moves on the spring network, initial conditions  were generated using a Voronoi diagram with 1000 randomly selected site points on a unit square.  The vertices on the boundary of the square were pinned, and the Tutte algorithm was applied to produce initial conditions for the graph considered as a network of springs.  This changed cells areas but not topological characteristics of the network such as the adjacency matrix.  In certain instances, we found that the Tutte algorithm sometimes produced vertices which were nearly incident.  To avoid numerical issues with computing angles,  each pair of vertices were identified when their distance was within a tolerance of $<10^{-6}$.  In a few cases, the number of vertices in a polygon were reduced to one or two points. Cell areas for both of these scenarios were nearly zero, and contribute almost nothing to weighted disorders. For unweighted measures,  we approximated the shape of polygons depending on the number of remaining points. 
In the case of two points, the polygon has a very high aspect ratio, and so we approximated the turning function by a ``$2$-gon" $R_2$, with a turning function given in (\ref{r2turn}).
For polygons $P_n$ which originally had $n$ sides before identifying vertices, we computed $d_2(R_2, R_n)$ using formula (\ref{mainregpoly}) for finding $\mathcal D$.  For hexagonal turning disorder, we computed $d_2(R_2, R_6) = \sqrt{6}\pi/9$, and for the circular disorder we computed $d_2(R_2, C) = \sqrt{3}\pi/6$.  In the case of one remaining point (whose appearance was quite rare in simulations, comprising about 0.2\% of cells), we assumed that the cell with originally $n$ sides was approximately a regular $n$-gon.  

In Fig. \ref{t1sim}, we show snapshots of a sample path of the network process with 3000 T1 moves applied to a network with 1000 cells.  The Tutte algorithm guarantees that all cells are convex. While cells are neither created nor destroyed from T1 moves, the microstructure visually appears to be coarsening as disparities of areas between cells increase.  Those cells with less area tend to have higher aspect ratios than larger cells (which also typically have larger numbers of sides).

In Fig. \ref{turn_plots}, we plot the number of T1 moves against the six turning disorders.  The most apparent difference between distances is that unweighted disorders increase at approximately constant rates, while weighted disorders remain mostly constant.  When restricted to either weighted or unweighted disorders, we have an ordering of $\mathcal D_{\mathrm{c}}< \mathcal D <\mathcal D_{6}$.  For unweighted disorders,  $\mathcal D_{\mathrm{c}}$ and $\mathcal D$ are nearly identical for low numbers of T1 moves, and about $10\%$ smaller than $\mathcal{D}_6$.  For larger amounts of moves, the average aspect ratio of cells increases, and $\mathcal D$ becomes more aligned with $\mathcal D_6$. This  phenomenon is reflected in Fig. \ref{ar_plot}, in which increasing aspect ratios results in larger values of the regular turning distance.

\begin{figure}
\centering
\includegraphics[width =\linewidth]{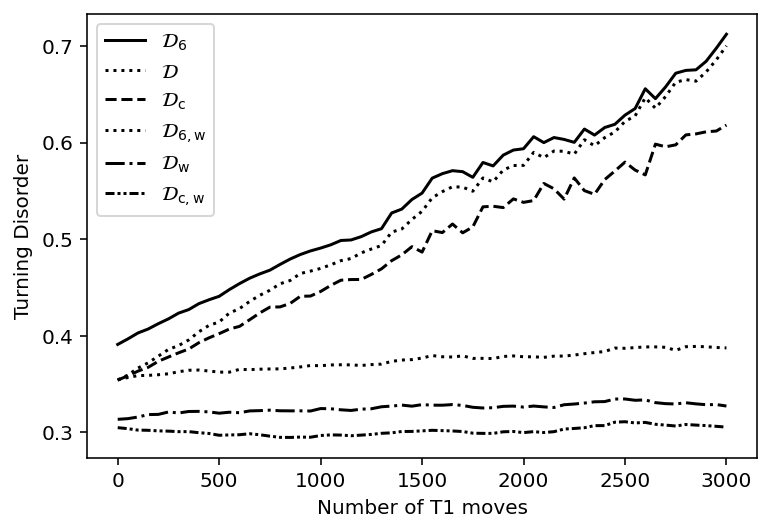}
	\caption{Turning disorders  for the T1 process over 3000 T1 moves.} \label{turn_plots}
\end{figure}

\subsection{Polygonal edge rupturing} \label{subsec:rupture}
\begin{figure}
\centering
	\includegraphics[width = \linewidth]{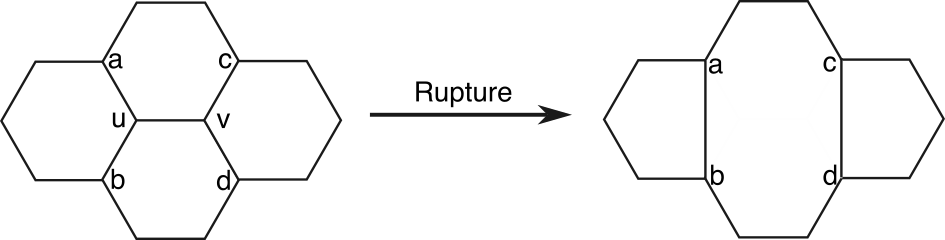}
	\caption{A polygonal edge rupture.} \label{rupture_move}
\end{figure}

For a stochastic process on networks which produces nonconvex polygons, we consider an example of rigid edge rupture studied in \cite{klobusicky2021markov}. Here, we take initial conditions of a hexagonal lattice and perform operations whose topological transitions are similar to rupturing an edge in a two-dimensional foam.  An example of polygonal edge rupture is given in Fig. \ref{rupture_move}. Here, the edge $(u,v)$ and all edges containing $u$ or $v$ are removed, and new edges $(a,b)$ and $(c,d)$ are created by inserting straight line segments.  In actual foams, cell walls have constant curvature and will immediately rearrange after a rupture to satisfy the Herring conditions, in which two edges meet at a vertex at $2 \pi/3$ radians (see \cite{klobusicky2024planar} for snapshots of the physical process).  For our purposes, there are two benefits of fixing vertex positions.  First, we find that performing multiple ruptures produces highly nonconvex  and polygonal cells, for which we confirm whether turning disorders produce large values.  Furthermore, the minimum distance between vertices in this process can only increase.  As opposed to the T1 process, we have no need to use approximations for when vertices are nearly incident.

Like the T1 processes in Section \ref{subsec:t1}, we denote $C_n$ as a cell with $n$ sides, and observe that a single polygonal rupture affects the side numbers of four cells $(C_i, C_j, C_k, C_l)$ adjacent to some edge.  The topological reactions are quite different, however, with
\begin{align}
	&C_i + C_j  \rightharpoonup C_{i +j - 4}, \qquad &\hbox{\textbf{(face merging)}}\\& C_k \rightharpoonup C_{k-1}, \quad  C_l \rightharpoonup C_{l-1}. &\hbox{\textbf{(edge merging)}}\label{react_rupture}
\end{align}
 The second-order ``face-merging" reaction $C_{i} + C_{j} \rightharpoonup C_{i +j - 4}$ is responsible for the generation of very large cells when the rupture operation is repeated, and carries similar ``gelation" properties to the simpler reaction $A_i + A_j \rightharpoonup A_{i+j}$ for the sticky particle dynamics described by the Smoluchowski coagulation equation \cite{klobusicky2021markov,aldous1999deterministic}.  As with the T1 process, we forbid edge ruptures which create cells with fewer than three sides.

\begin{figure}
\centering
	\includegraphics[width = \linewidth]{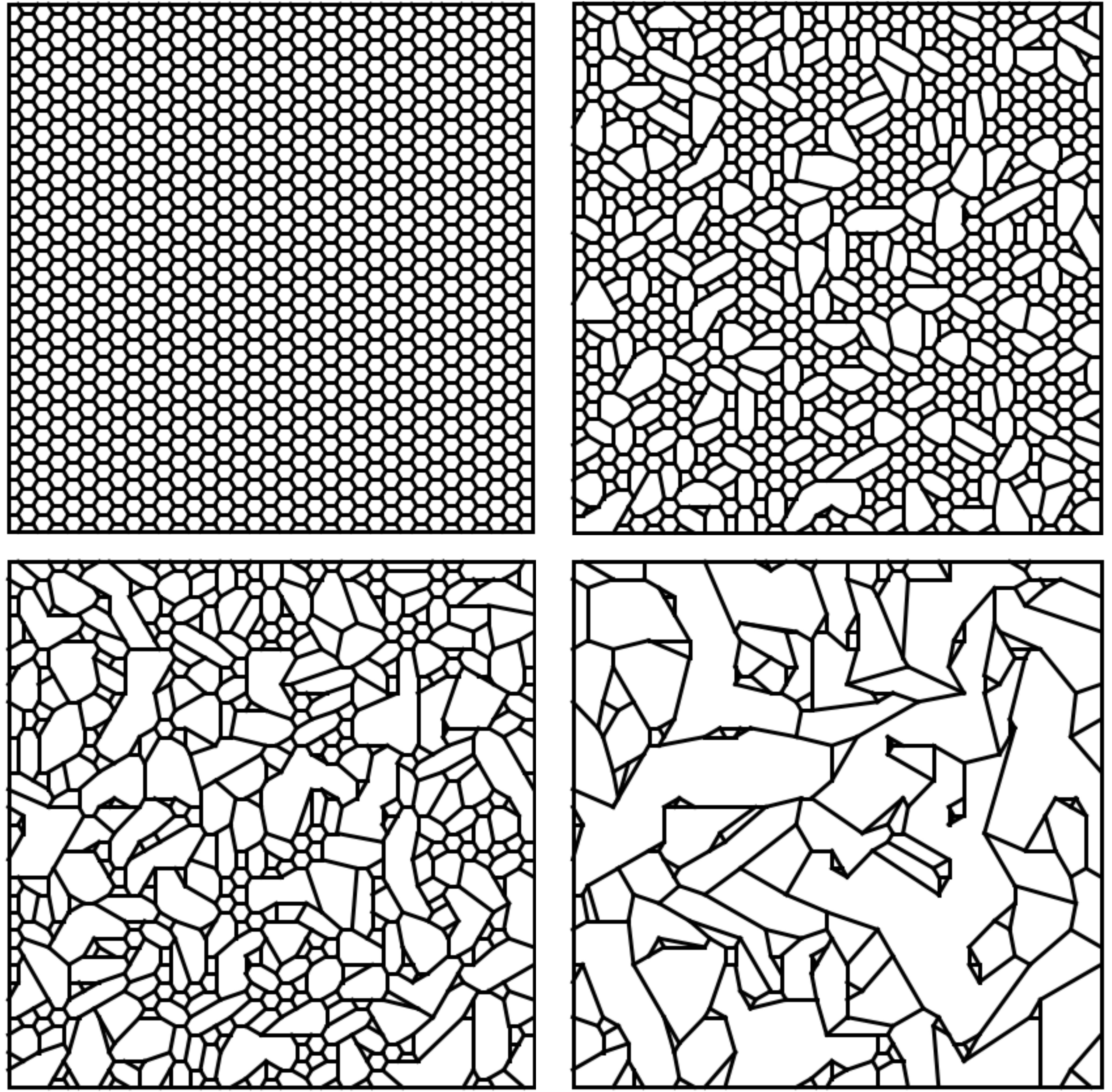}
	\caption{Snapshots of the rupture process with initial hexagonal lattice conditions. From top left to bottom right: the process after 0, 300, 600, and 900 ruptures.} \label{rupture_fig}
\end{figure}

In Fig. \ref{rupture_fig}, we show a sample path of the rupture process with 1067 initial cells and 900 ruptures.  Initially, the cells produced from ruptures remain mostly convex, but at around 500 ruptures more irregular shapes start to form.  At the end of the process, when about 15\% of cells remain, the vast majority of the foam is covered with highly nonconvex polygons, with relatively small 3- and 4-gons interspersed between them.  The plot in Fig. \ref{turn_plots} shows how each of the turning disorders changes as the number of ruptures increases.  Initially, both the hexagonal and turning disorders are nearly zero (boundary cells prevent this from being exactly zero), while the circular disorders are approximately $d_2(C, R_6) = \sqrt 3\pi/18 = 0.3023$.  As ruptures increase, weighted and unweighted disorders diverge, as larger cells tend to be more erratic.  After 500 ruptures, circular, hexagonal, and regular disorders look largely the same.  As expected, values for all disorders after 900 ruptures are larger than those found in the T1 process (Fig. \ref{rupture_plots}).

\begin{figure}
\centering
\includegraphics[width =\linewidth]{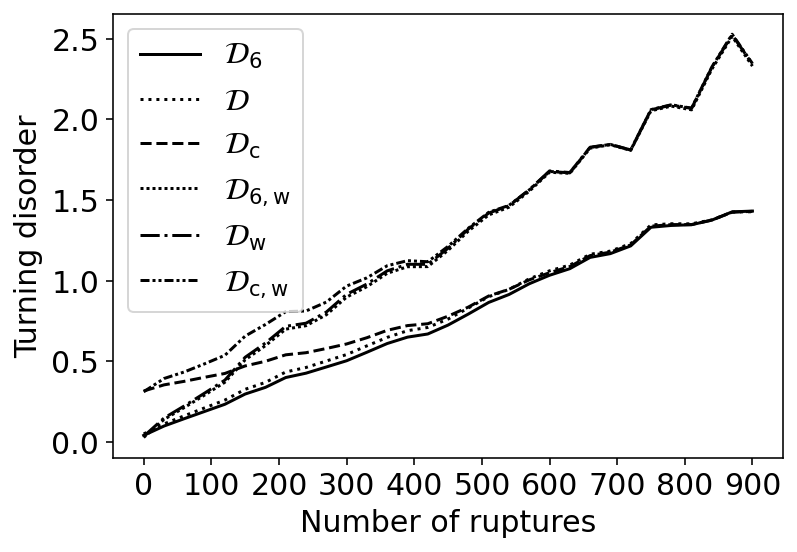}
\caption{Turning disorders for the polygonal rupturing process over 900 ruptures.} \label{rupture_plots}
\end{figure}

\section{Discussion} \label{sec:disc}

We have studied a computationally efficient method for measuring disorder in a planar network by taking averages of turning distances of each cell compared against ordered shapes.  From viewing several different examples of trivalent network microstructures, we found that turning disorders capture a wide range of geometric features.  For one example, we observed that Archimedean lattices are considered completely ordered when cells are compared against regular polygons, but have nonzero hexagonal and circular disorders which vary based on the types and relative areas of the lattice tiles.  From an analysis of the turning distance on regular polygons and circles, we were able to produce exact values for these lattices. We have also analyzed turning disorders for two classes of stochastic processes on geometric networks which are increasingly irregular. For aged spring networks under T1 moves, the network's area is populated with large, approximately round shapes, causing area-weighted disorders to remain relatively constant, while unweighted disorders increase due to the prevalence of smaller cells with higher aspect ratios. Large cells in aged rupture networks, on the other hand, are in general nonconvex, producing larger disorders which are approximately the same under all ordered shapes.   Several of the formulas derived in Sec. \ref{sec:compspec} were also used for the network processes.  In particular, we used formula (\ref{circturn}) in calculating distances between circles and general polygons.

The formulas (\ref{regdist})-(\ref{circdist}) for turning disorder do not consider the variation of turning distances across the network as an aspect of disorder.  One approach for producing an approximate empirical spatial disorder would sample disorders over small rectangular subsets of the microstructure.  A similar idea was employed in  Miley et al.\ \cite{miley2024metric} to describe a metric on the space of microstructures.  Here,  subimages of two microstructures, called \textit{windows}, are sampled and compared with the Wasserstein distance.  A spatial notion of disorder can also be paired with other geometric descriptors of microstructure in machine learning prediction problems.  Recent success toward this end have used graph neural networks (GNNs) for predicting magnetostriction \cite{dai2021graph} and the emergence of abnormal grain growth \cite{cohn2024graph}.  Datasets for these problems consist of the network's adjacency matrix along with a feature vector for geometric information such as grain size and misorientation. For future work, we hope to analyze how additional metrics like the turning disorders could enhance algorithms for predicting properties of materials.

\backmatter

\section{Declarations}

\textbf{Funding}: The work of all authors is partially supported  by the National Science Foundation under Grant No. 2316289. The work of J.K. and A.D is partially supported by NASA
under Grant No. 80NSSC20M0097.

\textbf{Conflicts of interest/Competing interests}:  The authors have no competing interests to declare that are relevant to the content of this article.

\textbf{Availability of data}: The datasets used and analyzed during the current study are available from the
corresponding author on reasonable request.

\textbf{Authors' contributions}: J.K. wrote the manuscript.  J.K., R.L, and B.S.   performed computational experiments.  A.D, J.K. and A.U. derived theoretical results.

\appendix

\section{Continuity under vertex perturbation}

A natural perturbation for a polygon $P$ is to adjust one of its vertices by a small amount.  We should expect for this perturbation to be continuous with respect to the $d_p$ metric.  This is indeed the case, and in fact the continuity is locally Lipschitz for $p = 1$.

\begin{proposition}
    \label{turnlip}
    Let $P$ be a unit-perimeter polygon with finitely many sides. Let $\tilde P$ be a polygon created from perturbing a vertex in $P$ by  a distance of $\varepsilon$.  Then 
    \begin{enumerate}
        \item For $p > 0$, the perturbation is continuous under the $p$-turning distance, meaning that 
        \begin{equation}
           d_p(P, \tilde P) \rightarrow 0 \quad \hbox{ as } \quad \varepsilon \rightarrow 0. 
        \end{equation}
        \item For $p= 1$, the continuity is locally Lipschitz for small perturbations. Thus, there exist $\varepsilon_0$ and $K>0$ such that for $\varepsilon < \varepsilon_0$, 
    \begin{equation}
        d_1(P,\tilde P) \le K\varepsilon.
    \end{equation} 
        
    \end{enumerate}
       
\end{proposition}
\begin{proof}

    Without loss of generality, we can assume that a polygon $P$ with $n \ge 3$ sides has an edge of length $\ell$ with one vertex at the origin and the other vertex $v$  on the positive $x$-axis.  We set the turning functions $f$ of $P$  and $\tilde f$  of $\tilde P$ to start tracing at the origin.

    We note that the right hand side of  (\ref{maindist}) can be bounded by setting $t = 0$ and $\theta = 0$, which is precisely the $L^p$ norm, meaning
    \begin{equation}
         d_p(f,g) \le   \|f-g\|_{p}. \label{dlesslp}
    \end{equation}
Continuity then follows from showing that turning functions vary continuously in $L^p$ from perturbing a single vertex.  To see this, we write the turning functions using indicator functions, with
\begin{equation}
f(s) = \sum_{i = 0}^{n-1} \theta_i \mathbf 1_{s\in [s_i, s_{i+1})}, \qquad \tilde f(s) = \sum_{i = 0}^{n-1} \tilde \theta_i \mathbf 1_{s\in [\tilde s_i, \tilde s_{i+1})},
\end{equation}
where $s_0 = \tilde s_0 = 0$, $s_n = \tilde s_n = 1$. Here, an indicator $\mathbf 1_A$ is equal to $1$ if $A$ holds, and is equal to $0$ otherwise.  Perturbing a single vertex only affects the turning angles of the two edges containing the vertex $v$, so $\theta_i = \tilde \theta_i$ for $ i \ge 2$.  Furthermore, a vertex perturbed by distance $\varepsilon$ can change its total perimeter, and also the distance between jump points in its turning function, by at most $2\varepsilon$.  We can then bound $|s_i- \tilde s_i| \le 2\varepsilon$ for $i = 2, \dots, n-1$.

To compare $\theta_0 = 0$ and $\tilde \theta_0$, we note that the greatest change in the turning angle occurs when $v = (\ell_1,0)$ is perturbed vertically to position $(\ell_1, \varepsilon)$.  Thus we can bound $|\theta_0-\tilde \theta_0| = |\tilde\theta_0| \le \tan^{-1}(\varepsilon/\ell_1)$.  Through an appropriate change of coordinates, a similar argument can be made to show $|\theta_1 - \tilde \theta_1| \le \tan^{-1}(\varepsilon/\ell_2)$.
Thus, we can bound the $L^p$ norm by
\begin{align}
\| f - \tilde f\|_p^p &\le 
\sum_{i = 1}^{n-1} |s_i - \tilde s_i| \|f- \tilde f\|_\infty ^p+  \ell_1|\theta_1 - \tilde \theta_1|^p+ \ell_2|\theta_2 - \tilde \theta_2|^p 
\\ &\le (n-1)(2\varepsilon)\|f- \tilde f\|_\infty ^p+ \ell_1\tan^{-1}(\varepsilon/\ell_1)^p+\ell_2\tan^{-1}(\varepsilon/\ell_2)^p.
\end{align}
Clearly polygons have turning functions with finite $L_\infty$ norms, and thus $\| f - \tilde f\|_p \rightarrow 0$ as $\varepsilon \rightarrow 0$. This shows continuity of perturbations for  $p\ge 1$.  When $p = 1$, for a sufficiently small $\varepsilon_0>0$ there exists a constant $K(P)>0$ such that for $\varepsilon< \varepsilon_0$, we have
\begin{align}
\| f - \tilde f\|_1 
\le (n-1)(2\varepsilon)(4 \pi)+ \ell_1\tan^{-1}(\varepsilon/\ell_1)+\ell_2\tan^{-1}(\varepsilon/\ell_2)  \le K \varepsilon,
\end{align}
where we use that $\tan^{-1}(x)/x \rightarrow 1$ as $x \rightarrow 0$.  This establishes local Lipschitz continuity.

\end{proof}

We conclude with a simple counterexample with $p>1$ where Lipschitz continuity does not hold.  Consider a class of right triangles $T_a$ with  legs of length $3$ and $a$.  These have perimeters of $P_a = 3+a+ \sqrt{9+a^2}$, and have turning functions of 
\begin{equation} \label{turncounter}
    f(s; a) = \begin{cases}
        0 & s \in[0, 3/P_a), \\
        \pi/2 & s \in [3/P_a, (3+a)/P_a), \\
        \pi+  \tan^{-1}(a/3) & s \in [(3+a)/P_a, 1).
    \end{cases}
\end{equation}

We will consider perturbations of the 3-4-5 triangle by comparing $T_4$ and $T_{4+\varepsilon}$ and finding a lower bound $d_p(T_4, T_{4+\varepsilon})$ which scales as $\varepsilon^{1/p}$, and therefore cannot be Lipschitz continuous for $p>1$.  

To show this bound, we begin by noting  that the length of the side for which $f(s; a) = \pi/2$ is equal to $\ell_a = a/P_a$.  Using a first-order Taylor expansion for the square root, the difference of this side length for $T_{4+\varepsilon}$ and $T_4$ is equal to 
\begin{align}
   m_\varepsilon &:=  \ell_{4+\varepsilon} - \ell_4 = \frac{4+\varepsilon}{P_{4+\varepsilon}} - \frac{4}{12} = \frac{12+3\varepsilon- P_{4+\varepsilon}}{3P_{4+\varepsilon}} \\
     &= \frac 13 \left(  \frac{5+2\varepsilon- \sqrt{25+ 8\varepsilon+\varepsilon^2}}{ 7+ \varepsilon+ \sqrt{25+ 8\varepsilon+\varepsilon^2} }        \right) \sim  \frac{\varepsilon}{30}.
\end{align}
Here, we write $f(\varepsilon) \sim g(\varepsilon)$ if $\lim_{\varepsilon \rightarrow 0} \frac fg = 1$.

\begin{figure}
\centering
\includegraphics[width =\linewidth]{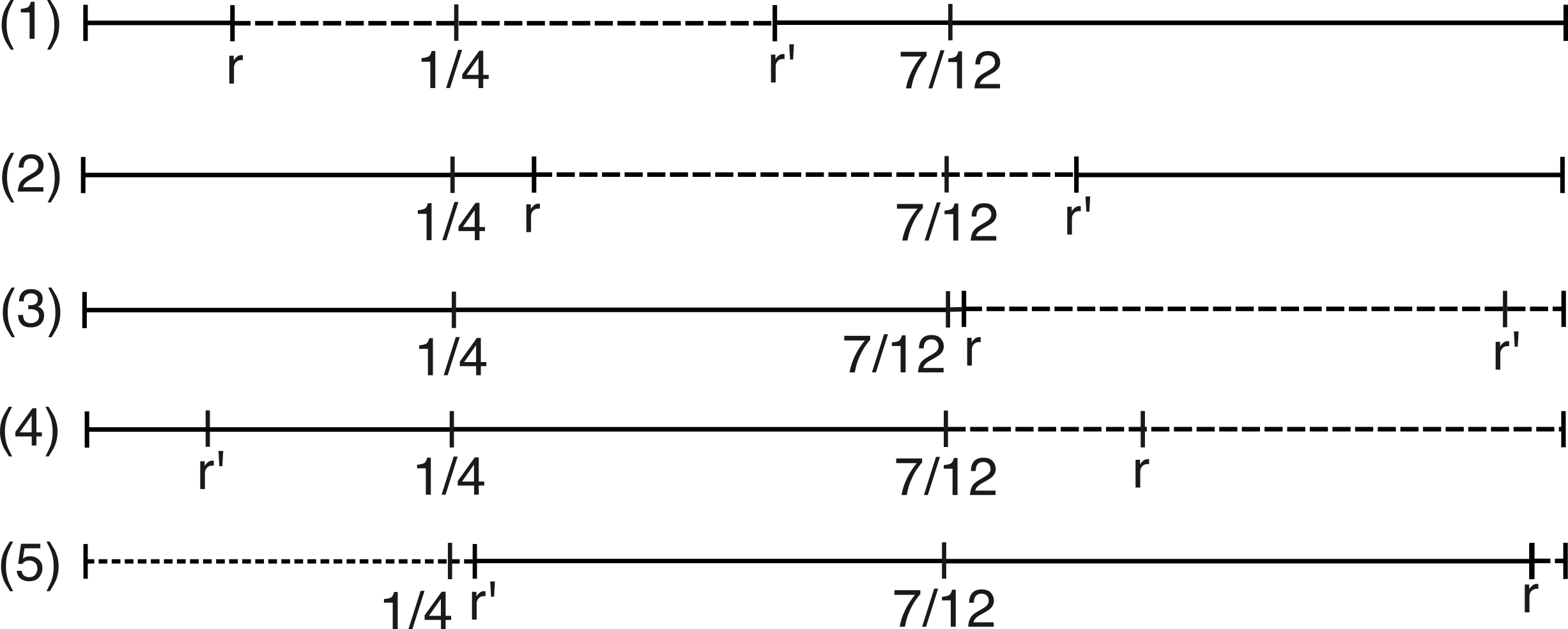}
\caption{Relative locations of jumps for varying values of $r$, where  (1) $r \in [0, 1/4- m_\varepsilon/2]$, (2) $r \in [1/4- m_\varepsilon/2, 7/12-m_\varepsilon/2]$, (3) $r \in [7/12-m_\varepsilon/2, 2/3-3m_\varepsilon/2]$, (4) $r \in [2/3-3m_\varepsilon/2, 1-m_\varepsilon/2]$, (5) $r \in [1-m_\varepsilon/2, 1]$.  Dashed lines denote intervals to the left and right of $q_r$ with length of at least $m_\varepsilon/2$.} \label{int_fig}
\end{figure}

For each $t\in [0,1]$, the set of jump points in $(0,1)$ for $F(s;t,\theta) := f(s;4)- f(s+t; 4+\varepsilon) + \theta$
is  $J_r = \{1/4, 7/12, r, r'\}$, where $r = t+3/P_{4+\varepsilon} \hbox{ mod  }1$ and $r' = r+\ell_{4+\varepsilon} \hbox{ mod  }1$.  We claim that for each $r\in [0,1]$ (and thus for each $t\in [0,1])$, there is a jump $q_r \in J_r$ such that (i) the jump size of $F$ is at least $\pi/4$ and (ii) the interval length of the steps on both sides of the jump is at least $m_\varepsilon/2$.  This can be directly verified by breaking into cases.  Indeed, 
\begin{enumerate}
\item  For $r \in [0, 1/4- m_\varepsilon/2]$,  $q_r = 1/4$.  
\item For $r \in [1/4- m_\varepsilon/2, 7/12-m_\varepsilon/2]$,  $q_r =  7/12$.
\item For $r \in [7/12-m_\varepsilon/2, 2/3-3m_\varepsilon/2]$  $q_r = r'$. Note that the right endpoint is where $r' = 1-m_\varepsilon/2$. 
\item For $r \in [2/3-3m_\varepsilon/2, 1-m_\varepsilon/2]$, $q_r = r$.
\item For $r \in [1-m_\varepsilon/2, 1]$,  $q_r = 1/4$.
\end{enumerate}

See Fig.  \ref{int_fig} for a visualization of the relative ordering of the jump points for cases (1)-(5).  In each case, the jump at $q_r$ corresponds to a jump in one and only one of $f(s;4)$ or $f(s+t;4+\varepsilon)$.  In either of these functions, one can verify from (\ref{turncounter}) that all jumps have size of at least  $\pi/4$. Thus, in at least one of the intervals to the left or right of $q_r$, $F$ must have a constant value of magnitude at least $\pi/8$.  Denoting this interval $I_{r(t)}$, we then bound the turning distance by
\begin{align}
d_p(f,\tilde f)  = \left( \min_{\theta \in \mathbb R, t \in [0,1]} \int_0^1 |F(s;t, \theta)|^p ds\right) ^ {1/p}
&\ge \left(\min_{\theta \in \mathbb R, t \in [0,1]}\int_{I_{r(t)}} \left(|F(s;t, \theta)|^p\right) ds \right)^{1/p} \\&\ge  \frac{\pi}{8}\left(\frac{m(\varepsilon)}2\right)^{1/p} \sim \frac{\pi}8 \left(\frac{\varepsilon}{60}\right)^{1/p}.
\end{align}

\bibliography{turn}

\end{document}